\documentclass[10pt,twocolumn,letterpaper]{article}

\usepackage{cvpr}              
\usepackage{times}
\usepackage{graphicx}  
\usepackage{amsmath}
\usepackage{bm}
\usepackage{natbib}

\usepackage{amsmath,amsfonts,bm, xcolor}

\def\eqref#1{equation~\ref{#1}}

\def\1{\bm{1}}

\def\vtheta{{\bm{\theta}}}

\DeclareMathAlphabet{\mathsfit}{\encodingdefault}{\sfdefault}{m}{sl}
\SetMathAlphabet{\mathsfit}{bold}{\encodingdefault}{\sfdefault}{bx}{n}

\usepackage{url}
\usepackage{algorithm}  
\usepackage{algorithmicx}    
\usepackage{algpseudocode} 
\usepackage{float}    
\usepackage{booktabs}
\usepackage{multirow}
\usepackage{svg} 
\usepackage{array}
\definecolor{mygrey}{gray}{0.6}

\usepackage{amsthm}
\usepackage{enumitem}
\usepackage{xspace}
\theoremstyle{plain}
\newtheorem{theorem}{Theorem}[section]

\theoremstyle{definition}

\newtheorem{assumption}[theorem]{Assumption}
\theoremstyle{remark}

\newcommand{\method}{\text{GRACE}\xspace}

\definecolor{cvprblue}{rgb}{0.21,0.49,0.74}
\usepackage[pagebackref,breaklinks,colorlinks,allcolors=cvprblue]{hyperref}

\def\paperID{}
\def\confName{CVPR}
\def\confYear{2026}

\title{Dynamic Feedback Engines: Layer-Wise Control for Self-Regulating Continual Learning}

\author{
    Hengyi Wu$^{1}$\thanks{Correspondence to: {\tt hw987@umd.edu}} \quad 
    Zhenyi Wang$^{2}$ \quad 
    Heng Huang$^{1}$ \\
    $^{1}$University of Maryland, College Park \quad $^{2}$University of Central Florida
}

\begin{document}
\maketitle
\begin{abstract}
Continual learning aims to acquire new tasks while preserving performance on previously learned ones, but most methods struggle with catastrophic forgetting. Existing approaches typically treat all layers uniformly, often trading stability for plasticity or vice versa. However, different layers naturally exhibit varying levels of uncertainty (entropy) when classifying tasks. High-entropy layers tend to underfit by failing to capture task-specific patterns, while low-entropy layers risk overfitting by becoming overly confident and specialized.
To address this imbalance, we propose an entropy-aware continual learning method that employs a dynamic feedback mechanism to regulate each layer based on its entropy. Specifically, our approach reduces entropy in high-entropy layers to mitigate underfitting and increases entropy in overly confident layers to alleviate overfitting. This adaptive regulation encourages the model to converge to wider local minima, which have been shown to improve generalization.
Our method is general and can be seamlessly integrated with both replay- and regularization-based approaches. Experiments on various datasets demonstrate substantial performance gains over state-of-the-art continual learning baselines. 

\end{abstract}    

\section{Introduction}

Continual learning \citep{Masana2023Class, DeLange2022} aims to enable models to acquire new tasks without suffering from catastrophic forgetting \citep{MCCLOSKEY1989109, French1999Catastrophic}, the degradation of performance on previously learned knowledge. This issue stems from the fundamental stability-plasticity dilemma; a model must be stable enough to preserve old knowledge while remaining sufficiently plastic to learn new information. To navigate this trade-off, research has predominantly explored three families of approaches: regularization-based, replay-based, and parameter-expansion methods. 

A primary limitation of many existing approaches is their reliance on simple regularization techniques, such as L1 or L2, which often guide the model toward sharp minima in the loss landscape, which are known to generalize poorly \citep{foret2021sharpnessawareminimizationefficientlyimproving}. This convergence to a narrow "deep well" rather than a broad "flat valley" is a form of overfitting that degrades test-time performance. Furthermore, these methods are typically layer-agnostic, lacking any mechanism to modulate learning based on layer-specific performance. By failing to preserve the knowledge in well-performing layers while simultaneously encouraging underperforming layers to adapt, they compromise the model's ability to balance stability and plasticity, ultimately hindering overall accuracy. 

Our primary contribution is a novel technique we call \textbf{Self-Adaptive Entropy Scaling}. While regularizing the final layer's classification entropy is a known and effective technique for improving model robustness \citep{cha2021cpr}, a naive, uniform application to all layers has a significant drawback. Indiscriminately penalizing layers that already exhibit high entropy can be detrimental, potentially degrading valuable learned representations. To address this, our entropy scaling method adaptively adjusts the regularization strength for each layer through the lens of Bayesian inference. The penalty is applied strongly to layers with low entropy (i.e., those with over-confident outputs) while being reduced for layers that already possess high-entropy features, preserving their diversity, as illustrated in Figure \ref{fig:placeholder}. 

To further enhance the efficacy of entropy scaling, we introduce a complementary adaptive training mechanism.
Our adaptive training modulates the plasticity of each layer based on its performance on previous tasks. Specifically, we constrain updates for high-performing layers to preserve their acquired knowledge, while conversely amplifying updates for underperforming layers to encourage more rapid adaptation. 
An important point is that our method does not disturb the natural tendency of earlier layers to learn general representations and later layers to learn task specific features. Our goal is not to force all layers to maintain identical entropy levels, but rather to prevent pathological cases where intermediate layers become either overly confident (low-entropy) or entirely uninformative (high-entropy). The entropy-based regularization acts as a soft constraint that preserves diversity and uncertainty calibration, not as a strict equalization across layers. Specifically, the regularization term is adaptive and data-dependent: it scales by the layer’s predictive variance, allowing layers to specialize while still maintaining a balanced uncertainty profile. Early layers are naturally expected to encode general, high-entropy representations, while deeper layers capture more discriminative, lower-entropy signals. Our formulation explicitly respects this hierarchy—by penalizing extreme deviations rather than entropy differences per se.

To demonstrate the effectiveness of our approach, we conduct both theoretical analysis and empirical evaluation. Theoretically, we show that our method leads to a tighter generalization error bound. Empirically, our experiments on standard image classification benchmarks confirm that the proposed approach significantly improves average accuracy while simultaneously reducing forgetting, outperforming state-of-the-art methods. Our contributions are as follows:
\begin{itemize}
   \item We propose a novel framework for continual learning that employs a dynamic feedback mechanism to apply layer-aware regularization, overcoming the limitations of layer-agnostic approaches.
\item We design a new algorithm that integrates two techniques, entropy scaling and adaptive training through Bayesian inference, to intelligently modulate plasticity across the network.
\item We conduct in-depth theoretical analysis that firmly supports the effectiveness of our method.
\item We conduct comprehensive experiments on popular continual learning datasets, achieving state-of-the-art results and showing marked improvements in both accuracy and knowledge retention.
\end{itemize}

\section{Related Work}

Continual learning (CL) addresses the challenge of training models on a sequence of tasks without catastrophically forgetting previously acquired knowledge. To this end, three primary classes of methods have been developed. Regularization-based approaches \citep{Rebuffi2017iCaRL, pmlr-v70-zenke17a-Synaptic, nguyen2018variational, aljundi2018memory, yan2024orchestrate} introduce penalty terms into the loss function to constrain updates on parameters critical for past tasks. Another line of work, memory-replay, maintains a buffer of exemplars from previous tasks \citep{shin2017continual, rolnick2019experience, lopez2017gradient, riemer2018learning, pham2021dualnet, arani2022learning, verwimp2021rehearsal} that are revisited during subsequent training to prevent knowledge degradation. A third approach, architecture expansion \citep{rusu2022progressiveneuralnetworks, 8578908PackNet, pmlr-v80-serra18a-Hard, li2019learn, hung2019compacting}, dynamically grows the network by adding new weights or adapters as new tasks arrive, thereby isolating task-specific parameters to prevent interference. 

While these general strategies are effective, recent works have explored output regularization-based approaches. For instance, CPR \citep{cha2021cpr} provides a strong baseline by regulating only the final output layer to find wider local minima. However, this singular focus means that the crucial intermediate layers are not explicitly regularized, potentially limiting their robustness against forgetting. A more closely related method, MOSE \citep{yan2024orchestrate}, addresses this by applying multi-level supervision to receive signals from all layers using a form of reverse self-distillation. Yet, their work treats each layer uniformly, overlooking the inherent functional differences between them. 

There are also works that investigate using entropy regularization in machine learning training. 
LegoGCD \citep{Cao2024Solving} selects known samples with high confidence to encourage the predictions of these samples to be closer to a uniform distribution. However they only consider the prediction accuracy of the final layer, whereas we attach classification heads to earlier layers and ensure their entropy is not too high or too low. CCL \citep{Wang2024Improving} uses distillation from less confident predictions to more confident predictions as an implicit form of entropy regularization. In contrast, our method applies direct regularization to reduce overconfidence.

In our work, we argue that not all layers are created equal. We bridge this gap by introducing a method that dynamically applies supervision across multiple layers, recognizing that earlier and later layers play distinct roles. This allows our model to reap the benefits of broad minima across the entire network, unlike CPR, while also leveraging the unique contributions of each layer. 

\section{Method}

\begin{figure}
    \centering
    \includegraphics[width=0.9\linewidth]{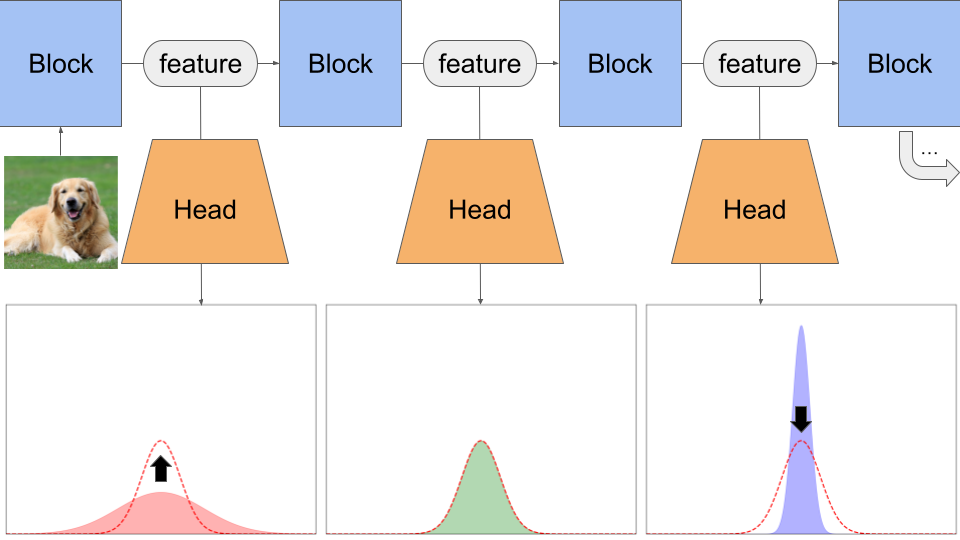}
    \caption{As the input data propagates through each successive block of the network, its feature vector at each layer is fed into a dedicated classification head. We calculate the entropy of the resulting output distribution from each head. This entropy measurement then dynamically adjusts the strength of the regularization term applied to that specific layer, allowing for adaptive regularization throughout the model. Without our method, layer output entropies (solid red, green, and blue) exhibit high variance, indicating that some layers become over-confident while others remain uncertain. Our method guides the entropy of each layer towards a stable, medium-entropy target (dashed red line), promoting more consistent representations throughout the network.}
    \label{fig:placeholder}
\end{figure}

\subsection{Preliminaries}
In the standard continual learning (CL) setup, a model is trained on a sequence of tasks, arriving one after another. Let the sequence of tasks be denoted by $\mathcal{T} = \{\mathcal{T}_1, \mathcal{T}_2, \ldots, \mathcal{T}_N\}$, where $N$ is the total number of tasks. Each task $\mathcal{T}_t$ for $t \in \{1, \ldots, N\}$ is associated with its own data distribution $\mathcal{D}_t = \{(\mathbf{x}_i, y_i)\}$, where $\mathbf{x}_i$ represents the input data and $y_i$ is the corresponding label. Except for a small memory buffer, the data from previous tasks $\mathcal{T}_{1}, \ldots, \mathcal{T}_{t-1}$ is not available when the model is learning the current task $\mathcal{T}_t$.

Our model is represented by a function $f(\cdot; \vtheta)$, parameterized by a set of parameters $\vtheta \in \mathbb{R}^d$. The goal of the model is to learn a mapping from inputs to outputs. Upon observing task $\mathcal{T}_t$, the model updates its parameters $\vtheta$ to minimize a task-specific loss function, $\mathcal{L}_t$. This loss is typically computed as the empirical risk over the data distribution $\mathcal{D}_t$:
\begin{equation}
\mathcal{L}_t(\vtheta) = \mathbb{E}_{(\mathbf{x}, y) \sim \mathcal{D}_t} [\ell(f(\mathbf{x}; \vtheta), y)]
\end{equation}

where $\ell(\cdot, \cdot)$ is a standard loss function, such as cross-entropy for classification.
Let $\vtheta_{t-1}^*$ denote the optimal parameters found after training on tasks up to $\mathcal{T}_{t-1}$. When task $\mathcal{T}_t$ arrives, the learning process aims to find a new set of parameters $\vtheta_t^*$ that minimizes $\mathcal{L}_t(\vtheta)$ without significantly increasing the loss on previous tasks. This is the core challenge of continual learning, known as catastrophic forgetting.
The ideal objective of a continual learning agent is to find a single set of parameters $\vtheta_N^*$ that performs well across all tasks simultaneously. This can be formulated as minimizing the total loss over the entire sequence:
\begin{equation}
\vtheta_N^* = \arg\min_{\vtheta} \sum_{t=1}^{N} \mathcal{L}_t(\vtheta)
\end{equation}

However, due to the sequential and constrained nature of data availability, achieving this joint optimization directly is not feasible. The primary goal of CL methods is to approximate this solution by sequentially updating the parameters $\vtheta$ in a way that balances performance on the current task with the preservation of knowledge from past tasks.
\subsection{\textbf{G}uided Ent\textbf{R}opy-\textbf{A}daptive Feedback for \textbf{C}ontinual L\textbf{E}arning (GRACE)}

Our proposed method, \method, is a framework designed to mitigate catastrophic forgetting in continual learning. The fundamental principle is to dynamically adjust the learning process for each layer based on two key signals: its current output entropy and its historical performance on past tasks. We formulate this as a general optimization problem where we adaptively scale the standard task and augment it with a regularization term that is entropy-scaled.

The general optimization objective for a given task is:
\begin{equation}
    \mathcal{L}(\vtheta) = \mathcal{L}_t(\vtheta) + R_t(\vtheta) = \sum _{l=1} ^L(\alpha_l \cdot \mathcal{L}_l(\vtheta) +  \gamma_l \cdot R_l(\vtheta))
\end{equation}
Here, $\mathcal{L}_l$ represents the primary objective function cross-entropy loss of layer $l$ on the current task, and it is modified by $\alpha_l$, an adaptive training modulator that scales the main loss term based on the layer's historical performance. For each layer, the regularization function $R_l$ is modulated by  $\gamma_l$, an entropy scaling factor. We sum over all the layers $1 \dots L$ to get the total task loss $\mathcal{L}_t$. The following sections detail the two core components of this framework: entropy scaling and adaptive training, which respectively define $\gamma_l$ and $\alpha_l$.

\textbf{Self-Adaptive Entropy Scaling ($\gamma_l$)}: The goal of entropy scaling is to encourage layers with over-confident (low-entropy) outputs to learn more generalizable representations, while protecting the features of layers that already exhibit high-entropy, diverse outputs. To achieve this, our method dynamically modulates the regularization penalty based on a layer's relative entropy compared to other layers within the same mini-batch, rather than its absolute entropy, as illustrated in Figure \ref{fig:placeholder}.

We propose a principled adaptive entropy scaling approach. We model the adaptive entropy scaling factor \( \gamma_\ell \) in layer \( \ell \) as a latent variable in a Bayesian framework. The goal is to infer a suitable regularization strength for each layer, based on its entropy \( H_\ell \). We use variational inference to approximate the posterior distribution over \( \gamma_\ell \) given the observed entropy.
Let \( \gamma_\ell \in \mathbb{R}_+ \) be the entropy regularization strength for layer \( \ell \), and \( H_\ell \in \mathbb{R}_+ \) the entropy observed at that layer. We assume a Gaussian noise model for entropy given regularization: $H_\ell = H^* + \frac{c}{\gamma_\ell} + \varepsilon_\ell, \quad \varepsilon_\ell \sim \mathcal{N}(0, \sigma^2)$

Where $H^*$
 is a target entropy value -- a reference or anchor that represents the desired entropy level for a model layer. We do not need to estimate $H^*$ since our approach can normalize the entropy values across different layers without knowing $H^*$ by treating the expectation of $H$ as $H^*$, illustrated in Eq. (\ref{eq:norm}). $c$ is a constant and $\sigma$ is the standard deviation. Thus, the likelihood becomes:
\[
p(H_\ell \mid \gamma_\ell) = \mathcal{N}\left(H_\ell \mid H^* + \frac{c}{\gamma_\ell}, \sigma^2 \right)
\]

We place a log-normal prior on \( \gamma_\ell \):
\[
p(\gamma_\ell) = \text{LogNormal}(\mu_0, \tau^2) \quad \Rightarrow \quad \log \gamma_\ell \sim \mathcal{N}(\mu_0, \tau^2)
\]
A log-normal prior is chosen for \( \gamma_\ell \) because it ensures positivity, as \( \gamma_\ell > 0 \) by design. It naturally models multiplicative uncertainty, which is appropriate for scaling factors in entropy regularization. The distribution also has heavy tails, allowing the model to flexibly assign both strong and weak regularization across layers. Furthermore, operating in log-space-- where \( \log \gamma_\ell \sim \mathcal{N} \)-- enables efficient variational inference via the reparameterization trick and permits closed-form KL divergence computation.
The posterior over \( \gamma_\ell \) is intractable, so we approximate it via variational inference. Let the variational posterior be:
\[
q_\phi(\gamma_\ell) = \text{LogNormal}(\mu_\phi, \sigma_\phi^2) \quad \Rightarrow \quad \log \gamma_\ell \sim \mathcal{N}(\mu_\phi, \sigma_\phi^2)
\]

We optimize the evidence lower bound (ELBO):
\[
\mathcal{L}(\phi) = \mathbb{E}_{\gamma_\ell \sim q_\phi} \left[
\log p(H_\ell \mid \gamma_\ell)
+ \log p(\gamma_\ell)
- \log q_\phi(\gamma_\ell)
\right]
\]

Therefore, the ELBO becomes:
\begin{align*}
&\mathcal{L}(\phi)=\mathbb{E}_{\gamma_\ell \sim q_\phi} \bigg[
-\frac{1}{2\sigma^2} \left(H_\ell - H^* - \frac{c}{\gamma_\ell} \right)^2  \\
& -\frac{1}{2\tau^2} (\log \gamma_\ell - \mu_0)^2 + \frac{1}{2\sigma_\phi^2} (\log \gamma_\ell - \mu_\phi)^2 
\bigg] + \text{const}
\end{align*}
To avoid optimizing \( \mathcal{L}(\phi) \) explicitly during training, we approximate the \textit{posterior mean}:
\begin{equation}
    \hat{\gamma}_\ell = \mathbb{E}_{q_\phi}[\gamma_\ell] = \exp\left( \mu_\phi + \frac{\sigma_\phi^2}{2} \right)
\end{equation}
Assuming a small variance \( \sigma_\phi^2 \approx 0 \), we approximate $\hat{\gamma}_\ell \approx \exp(\mu_\phi)$. Empirically, we set:
\begin{equation} \label{eq:norm}
    \mu_\phi \approx \tanh\left( z_\ell \right), \quad \text{where } z_\ell = \frac{H_\ell - \mu_H}{\sigma_H}
\end{equation}
For each mini-batch, we first compute the average output entropy $\bar{H}_l$ for every layer $l$. The $\mu_H$ denotes the mean of this set of entropies $\{\bar{H}_1, \dots, \bar{H}_{L}\}$ and $\sigma_H$ denotes the corresponding standard deviation. $z_\ell$ (z-score) denotes the relative entropy for each layer. This z-score, which measures how far a layer's entropy deviates from the batch average, is used to compute the final scaling factor $\gamma_l$. This yields the approximation used in GRACE:
\begin{equation}
    \gamma_\ell \approx \exp\left( \tanh\left( \frac{H_\ell - \mu_H}{\sigma_H} \right) \right)
    \label{eq:gamma}
\end{equation}
For our experiments, the right side of this equation is multiplied by a hyperparameter $\beta$. This formulation ensures an inverse relationship between relative entropy and the regularization strength. A layer with lower-than-average entropy (negative $z_l$) will yield a scaling factor $\gamma_l < 1$, promoting higher entropy to reduce overconfident predictions. Conversely, a layer with higher-than-average entropy (positive $z_l$) will receive a scaling factor $\gamma_l > 1$, thereby strengthening the regularization effect to reduce entropy and alleviate underfitting.
\begin{algorithm}[H]
    \caption{Training with Layer-wise Adaptive Regularization}
    \label{alg:my_algorithm_final}
    \begin{algorithmic}[1]
        \Require ResNet model $f(\cdot; \vtheta)$ with $L$ layers. Sequence of training datasets $D_1, \dots, D_T$. Validation set $D_{val}$. Learning rate $\eta$.
        \State Initialize modulators $\alpha_l \gets 1$ for all $l \in \{1, \dots, L\}$.
        \For{task $t \gets 1$ to $T$} \Comment{Adaptive Training: Calculate $\alpha_l$ for the current task }
            \If{$t > 1$}
                
                \State Let $\mathcal{A}_{\text{set}}$ be an empty set
                \For{$l \gets 1$ to $L$}
                    \State $A_l \gets \text{EvaluateAccuracy}(f(\cdot; \vtheta)_l, D_{val})$ \Comment{Evaluate layer on past data}
                    \State Add $A_l$ to $\mathcal{A}_{\text{set}}$
                \EndFor
                \State $\mu_A \gets \text{Mean}(\mathcal{A}_{\text{set}})$
                \quad $\sigma_A \gets \text{StdDev}(\mathcal{A}_{\text{set}})$
                \quad $\mathcal{L}_t \gets 0$
                \For{$l \gets 1$ to $L$}
                    \State $s_l \gets (A_l - \mu_A) / \sigma_A$
                    \State $\alpha_l \gets e^{\tanh(-s_l)}$ \Comment{Calculate the modulator based on the score}
            
                     \State $\mathcal{L}_t \gets \mathcal{L}_t +  \alpha_l \cdot \mathcal{L}_l$ 
                \EndFor
            \EndIf
            \For{each training epoch} \Comment{Train on current task $t$ }
                \For{each mini-batch $(X_b, Y_b) \in D_t$}
                    \State Perform forward pass to get activations $h_l$ for each layer $l$.
                    \State Let $\mathcal{H}_{\text{set}}$ be an empty list
                    \For{$l \gets 1$ to $L$}
                        \State $p_l \gets \text{softmax}(h_l)$,
                        \quad $H(p_l) \gets - \sum_{i} p_{l,i} \log p_{l,i}$
                        \State $\bar{H}_l \gets \mathbb{E}_{(\mathbf{x},y) \in (X_b, Y_b)}[H(p_l)]$ \Comment{Average entropy over batch}
                        \State Add $\bar{H}_l$ to $\mathcal{H}_{\text{set}}$
                    \EndFor
                    \State $\mu_H \gets \text{Mean}(\mathcal{H}_{\text{set}})$, $\sigma_H \gets \text{StdDev}(\mathcal{H}_{\text{set}})$, $\mathcal{R}_t \gets 0$
                   
                    \For{$l \gets 1$ to $L$}
                        \State $z_{l} \gets (\mathcal{H}_l - \mu_H) / \sigma_H$  
                        \State $\gamma_l \gets \beta e^{\tanh(z_l)}$ 
                        \Comment{Calculate scaling factor based on z-score}
                        \State $\mathcal{R}_t \gets \mathcal{R}_t +  \gamma_l \cdot \bar{H}_l$ 
                    \EndFor
                    \State $\mathcal{L} \gets \mathcal{L}_t + \mathcal{R}_t$
                    \State Update parameters: $\vtheta \gets \vtheta - \eta \nabla_\vtheta \mathcal{L}$.
                \EndFor
            \EndFor
            \State Update $D_{val}$ with representative samples from current task $t$.
        \EndFor
    \end{algorithmic}
\end{algorithm}

\textbf{Adaptive Training ($\alpha_l$)}: Inspired by our proposed Self-Adaptive Entropy Scaling, the principle of adaptive training is to dynamically adjust each layer's plasticity based on its performance on previously seen tasks. This allows us to preserve knowledge in stable, well-performing layers while encouraging adaptation in underperforming ones.
This is implemented via the learning modulator $\alpha_l$. After completing training on a task, we evaluate the average accuracy $A_l$ for each layer on a validation set of past tasks. We then quantify the relative performance of each layer by calculating its z-score,  which measures the deviation from the mean accuracy ($\mu_A$) in units of standard deviation ($\sigma_A$). (Here we use the variable $s$ for "score" to differentiate it from the z-score used in entropy scaling):
$s_l = \frac{A_l - \mu_A}{\sigma_A}$.
This z-score is then mapped to the modulator $\alpha_l$ for the \textit{next} training task. This mapping is designed such that high-performing layers ($s_l > 0$) receive a smaller $\alpha_l$ (e.g., $< 1$), reducing the impact of regularization and thus preserving their weights. Conversely, underperforming layers ($s_l < 0$) receive a larger $\alpha_l$ (e.g., $> 1$), emphasizing their importance. Finally, we calculate $\alpha_l$ as follows, to bound it within a reasonable range around 1: 
\begin{equation}
    \alpha_l = e^{\tanh(-s_l)} 
    \label{eq:alpha}
\end{equation}
We use $\alpha _l$ to modify this $\mathcal{L}_l$ (the loss for the classification head attached to layer $l$ on the current task's data). 
We then present the detailed algorithm in Algorithm \ref{alg:my_algorithm_final}.

\section{Theoretical Analysis}

We perform theoretical analysis about the generalization error with our adaptive entropy control in Theorem \ref{theorythm:pacbayes-entropy} and forgetting bound in Theorem \ref{theorythm:forgetting-drift}.

\begin{assumption}{(Smoothness).}
Each population objective $\mathcal{L}_t$ is $L_t$-smooth; i.e., for all $\vtheta,\vtheta'$,
\[
\big\|\nabla \mathcal{L}_t(\vtheta)-\nabla \mathcal{L}_t(\vtheta')\big\|
\;\le\;
L_t\,\|\vtheta-\vtheta'\|.
\]
\end{assumption}

\begin{assumption}{ (Entropy Lipschitzness).}
Along the training trajectory, each layer-entropy is Lipschitz in the parameters: for all $\vtheta,\vtheta'$ and each layer $\ell$,
\[
| H_{l}(\vtheta) - H_l(\vtheta')|
\;\le\;
c_\ell\,\|\vtheta-\vtheta'\|.
\]
\end{assumption}

\begin{assumption}{(Posterior concentration).}
The training algorithm induces a posterior $q_t$ with finite second moment such that
\[
\mathbb{E}_{\vtheta\sim q_t}\big\|\vtheta-\bar{\vtheta}_t\big\|^{2}
\;\le\;
\sigma_t^{2},
\]
where $\bar{\vtheta}_t := \mathbb{E}_{\vtheta\sim q_t}[\vtheta]$.
\end{assumption}

\paragraph{Cumulative entropy deviation.}
We define the cumulative (layerwise) entropy deviation at task $t$:
\[
\Delta_t
\;:=\;
\sum_{\ell=1}^{L}
\Big( H_l(\vtheta_t) - H^{\ast}_{\ell,t} \Big)^{2}.
\]
where $H^{\ast}_{\ell,t}$ denotes the target entropy for layer $l$ at task $t$.

\begin{theorem}[PAC-Bayes generalization with entropy control]
\label{theorythm:pacbayes-entropy}
Fix $\delta\in(0,1)$. For task $t$ with sample size $n_t\ge 2$, let
\[
\mathcal{L}_t(\vtheta)\;=\;R_t(\vtheta)\;+\;\lambda\,\Delta_t(\vtheta)
\]
\[
R_t(\vtheta)\;=\;\mathbb{E}_{(x,y)\sim\mathcal{D}_t}\!\big[\ell_t(f_\vtheta(x),y)\big],
\]
where $\ell_t\in[0,1]$ and $\Delta_t(\vtheta)=\sum_{\ell=1}^L \big(H_l(\vtheta)) - H_{\ell,t}^{*}\big)^2$.

Let $\hat{\mathcal{L}}_t(\vtheta)=\hat{R}_t(\vtheta)+\lambda\,\hat{\Delta}_t(\vtheta)$ be the empirical analogue on $n_t$ samples.
For any posterior $q_t$ absolutely continuous w.r.t.\ a prior $p_t$ (chosen before seeing the task-$t$ data), with probability at least $1-\delta$ over the sample,
\begin{align}
   & \mathbb{E}_{\vtheta\sim q_t}\big[\mathcal{L}_t(\vtheta)\big]
\; \le\;
\mathbb{E}_{\vtheta\sim q_t}\big[\hat{\mathcal{L}}_t(\vtheta)\big]\\ \nonumber
&
\;+\;
\sqrt{\frac{\mathrm{KL}(q_t\Vert p_t)+\ln\frac{2\sqrt{n_t}}{\delta}}{2(n_t-1)}}
\;+\; 
\lambda\,\mathbb{E}_{\vtheta\sim q_t}\big[\Delta_t(\vtheta)\big].
\label{theoryeq:main-pac}
\end{align}

Moreover, if $p_t=\mathcal{N}(\vtheta_{t-1},\Sigma_p)$ and $q_t$ has mean $\bar\vtheta_t$ and covariance $\Sigma_q$,
\[
\mathrm{KL}(q_t\Vert p_t)
\;\le\;
\frac{\kappa_t}{2}\,\|\bar\vtheta_t-\vtheta_{t-1}\|^2
\;+\;C_t,
\;\;
\kappa_t:=\lambda_{\max}(\Sigma_p^{-1}),\;\;
\]
\[
C_t:=\tfrac{1}{2}\!\left(\mathrm{tr}(\Sigma_p^{-1}\Sigma_q)-k+\ln\frac{\det\Sigma_p}{\det\Sigma_q}\right).
\]
\end{theorem}

\begin{theorem}[Forgetting bound via parameter drift]
\label{theorythm:forgetting-drift}
Let $\mathcal{F}_{s\to t}:=\mathcal{L}_s(\vtheta_t)-\mathcal{L}_s(\vtheta_s)$ for $1\le s<t\le T$. 
Assume:

\begin{enumerate}[label=(A\arabic*)]
\item \label{theoryA:boundedloss} $\ell_s\in[0,1]$ and along the optimization trajectory the population objective $\mathcal{L}_s$ has bounded gradient: $\sup_{\vtheta\in\Gamma}\|\nabla \mathcal{L}_s(\vtheta)\|\le L_s$, where $\Gamma$ contains $\{\vtheta_k\}_{k=s}^t$ and the line segments between successive iterates.
\item \label{theoryA:entropy-lip} (Entropy Lipschitzness) For each layer $\ell$, $H(Z_\ell(\vtheta))$ is Lipschitz in $\vtheta$ with constant $c_\ell$ along $\Gamma$.
\item \label{theoryA:local-strong} (Local strong convexity/PL for task $k$) The empirical objective $J_k(\vtheta):=\hat{\mathcal{L}}_k(\vtheta)=\hat{R}_k(\vtheta)+\lambda\,\hat{\Delta}_k(\vtheta)$ is $\mu_k$-strongly convex on the segment between $\vtheta_{k-1}$ and $\vtheta_k$, i.e.,
\[
\big\langle \nabla J_k(\vtheta)-\nabla J_k(\vtheta'),\,\vtheta-\vtheta'\big\rangle \;\ge\; \mu_k\,\|\vtheta-\vtheta'\|^2.
\]
\end{enumerate}
Then
\begin{equation}
    \mathcal{F}_{s\to t}\le
L_s\,\sum_{k=s+1}^{t}
\frac{1}{\mu_k}\Big(
\big\|\nabla \hat{R}_k(\vtheta_{k-1})\big\|
+
2\lambda\,C_\Delta\,\sqrt{\hat{\Delta}_k(\vtheta_{k-1})}
\Big),
\label{theoryeq:forgetting-final}
\end{equation}
\[
C_\Delta:=\Big(\sum_{\ell=1}^L c_\ell^2\Big)^{1/2}.
\]
Equivalently, since $\|\nabla \hat{R}_k\|\le \|\nabla \hat{\mathcal{L}}_k\|+2\lambda C_\Delta\sqrt{\hat{\Delta}_k}$,
\[
\mathcal{F}_{s\to t}
\le
L_s\,\sum_{k=s+1}^{t}
\frac{1}{\mu_k}\Big(
\big\|\nabla \hat{\mathcal{L}}_k(\vtheta_{k-1})\big\|
+
4\lambda\,C_\Delta\,\sqrt{\hat{\Delta}_k(\vtheta_{k-1})}
\Big).
\]
\end{theorem}

Due to space limitations, we provide theorem proof in Appendix.
\section{Experiment}

\subsection{Experiment Setup}

\textbf{Datasets}: We evaluate our method on four benchmark datasets: CIFAR-10 , CIFAR-100, Tiny-ImageNet, and CUB200. Following standard class-incremental learning protocols, we partition each dataset into a sequence of distinct tasks. Specifically, CIFAR-10 is divided into 5 tasks of 2 classes each, CIFAR-100 is split into 10 tasks of 10 classes each, Tiny-ImageNet is partitioned into 10 tasks of 20 classes each, and CUB200 is split into 10 tasks of 20 classes each. 

\paragraph{Baselines}: We compare our method to strong baselines, including
AGEM \citep{AGEM}, ER \citep{chaudhry2019tiny}, MIR \citep{aljundi2019online}, GSS \citep{aljundi2019gradient}, ASER \citep{shim2021online}, ER-AML \citep{caccia2022new}, GDumb \citep{prabhu2020greedy}, SCR \citep{mai2021supervised}, OCM \citep{pmlr-v162-guo22g}, OnPro \citep{wei2023online}, GSA \citep{guo2023dealing}, DER++ \citep{buzzega2020dark},
IL2A \citep{zhu2021classincremental}, CO2L \citep{Cha_2021_ICCV}, LUCIR \citep{Hou_2019_CVPR}, CCIL \citep{Mittal_2021_CVPR}, BIC \citep{wu2019large},  SSIL \citep{Ahn_2021_ICCV}, and MOSE \citep{yan2024orchestrate}.

\paragraph{Implementation details}:
For our experiments, baseline results for all methods were adapted from \citep{yan2024orchestrate}. The only exception was the MOSE baseline itself, which we reproduced to ensure a fair comparison.  On the Tiny-ImageNet dataset, our reproduction using the unmodified official code yielded between 2 to 6\%  higher accuracy than the results reported in the original paper. We therefore average these baselines for all subsequent comparisons. All experiments were conducted on a single NVIDIA RTX 2080 Ti GPU with 12GB of VRAM, except for Tiny-ImageNet (with a buffer size of 10,000) required an NVIDIA RTX A4000 with 16GB of RAM. Each reported result is the mean and standard deviation computed over 10 independent runs. We use ResNet18 with a random initialization, and the model was trained in the online scenario for 1 epoch per task to make a fair comparison with the other methods. We use batch size 10 and buffer batch size 64. We use the Adam optimizer, with learning rate $10^{-3}$ and weight decay $10^{-4}$. We set the hyperparameter $\beta$ to 0.005. We regularize all 4 of the layers with each layer having its own classification head. We train the model from scratch and use cross-entropy loss. We evaluate forgetting with backward transfer (BWT).

\subsection{Results}
Our proposed method demonstrates significant improvements over existing state-of-the-art approaches in continual learning, as shown in Table \ref{tab:comprehensive_results} (overall accuracy) and Table \ref{tab:main_results_af_only} (forgetting). On the Split CIFAR-100 benchmark, our method outperforms the strongest baseline by up to 2.0\% in average accuracy, and the performance gains are even more pronounced on the more challenging Split Tiny-ImageNet dataset, where our method achieves up to 3.4\% higher average accuracy. When compared to the second-most accurate method, our approach reduces the forgetting metric by an average of 1.1\% on CIFAR-100 and 6.6\% on Tiny-ImageNet. While some methods focusing exclusively on mitigating forgetting may report lower forgetting values in isolation, they do so at a significant cost to overall accuracy, making our method the most effective and practical solution. We attribute these gains to our dynamic, layer-aware regularization strategy, which contrasts with the static approaches common in prior work. The improvement in overall accuracy is primarily driven by the entropy scaling component. By selectively penalizing over-confident layers, our method effectively mitigates overfitting, a conclusion supported by observing higher validation accuracy despite lower training accuracy compared to baselines. Concurrently, the reduction in catastrophic forgetting stems from the adaptive training mechanism. By constraining updates to well-performing layers, this component successfully preserves previously acquired knowledge. Furthermore, we validate the practical utility of our method in resource-constrained environments in Table \ref{tab:memory_efficiency}, demonstrating its strong performance in small-buffer scenarios characteristic of online learning. 

\begin{table*}[t]
    \centering
    \caption{Comprehensive comparison of continual learning methods on Split CIFAR-100 and Split Tiny-ImageNet under various memory constraints. All values are Accuracy (\%). M is the buffer size.}
    \label{tab:comprehensive_results}
    \resizebox{0.95\textwidth}{!}{%
        \begin{tabular}{@{}lcccccc@{}}
            \toprule
            \multirow{2}{*}{\textbf{Method}} & \multicolumn{3}{c}{\textbf{Split CIFAR-100 (10 tasks) - ACC(\%) $\uparrow$}} & \multicolumn{3}{c}{\textbf{Split Tiny-ImageNet (100 tasks) - ACC(\%) $\uparrow$}} \\
            \cmidrule(lr){2-4} \cmidrule(lr){5-7}
            & \textbf{M = 1k} & \textbf{M = 2k} & \textbf{M = 5k} & \textbf{M = 2k} & \textbf{M = 4k} & \textbf{M = 10k} \\
            \midrule
            AGEM (2019)     & 5.8$\pm$0.2  & 5.9$\pm$0.3  & 6.1$\pm$0.4  & 0.9$\pm$0.1  & 2.0$\pm$0.5  & 3.9$\pm$0.2  \\
            ER (2019)       & 15.7$\pm$0.3 & 21.3$\pm$0.5 & 28.8$\pm$0.8 & 4.7$\pm$0.5  & 10.1$\pm$0.7 & 11.7$\pm$0.2 \\
            MIR (2019)      & 16.0$\pm$0.4 & 19.0$\pm$0.1 & 24.1$\pm$0.2 & 6.1$\pm$0.5  & 11.7$\pm$0.2 & 13.5$\pm$0.2 \\
            GSS (2019)      & 11.1$\pm$0.2 & 13.3$\pm$0.5 & 17.4$\pm$0.1 & 3.3$\pm$0.5  & 10.0$\pm$0.2 & 10.5$\pm$0.2 \\
            ASER (2021)     & 16.4$\pm$0.3 & 12.2$\pm$1.9 & 27.1$\pm$0.3 & 5.3$\pm$0.3  & 8.2$\pm$0.2  & 10.3$\pm$0.4 \\
            ER-AML (2022)   & 16.1$\pm$0.4 & 17.6$\pm$0.5 & 22.6$\pm$0.1 & 5.4$\pm$0.2  & 7.1$\pm$0.5  & 10.1$\pm$0.4 \\
            GDumb (2020)    & 17.1$\pm$0.4 & 25.1$\pm$0.2 & 38.6$\pm$0.5 & 12.6$\pm$0.1 & 12.7$\pm$0.3 & 15.7$\pm$0.2 \\
            SCR (2021)      & 27.3$\pm$0.4 & 30.8$\pm$0.5 & 36.5$\pm$0.3 & 12.6$\pm$1.1 & 18.2$\pm$0.1 & 21.1$\pm$1.1 \\
            OCM (2022)      & 28.1$\pm$0.3 & 35.0$\pm$0.4 & 42.4$\pm$0.5 & 15.7$\pm$0.2 & 21.2$\pm$0.4 & 27.0$\pm$0.3 \\
            OnPro (2023)    & 30.0$\pm$0.4 & 35.9$\pm$0.6 & 41.3$\pm$0.5 & 16.9$\pm$0.4 & 22.1$\pm$0.4 & 29.8$\pm$0.5 \\
            GSA (2023)      & 31.4$\pm$0.2 & 39.7$\pm$0.6 & 49.7$\pm$0.2 & 18.4$\pm$0.4 & 26.0$\pm$0.2 & 33.2$\pm$0.4 \\
            DER++ (2020)    & 15.3$\pm$0.2 & 19.7$\pm$1.5 & 27.0$\pm$0.7 & 4.5$\pm$0.3  & 10.1$\pm$0.3 & 17.6$\pm$0.5 \\
            IL2A (2021)     & 18.2$\pm$1.2 & 19.7$\pm$0.5 & 22.4$\pm$0.2 & 5.5$\pm$0.7  & 8.1$\pm$1.2  & 11.6$\pm$0.4 \\
            Co2L (2021)     & 17.1$\pm$0.4 & 24.2$\pm$0.2 & 32.2$\pm$0.5 & 10.1$\pm$0.2 & 15.8$\pm$0.4 & 22.5$\pm$1.2 \\
            LUCIR (2019)    & 8.6$\pm$1.3  & 19.5$\pm$0.7 & 16.9$\pm$0.5 & 7.6$\pm$0.5  & 9.6$\pm$0.7  & 12.5$\pm$0.7 \\
            CCIL (2021)     & 18.5$\pm$0.3 & 19.1$\pm$0.4 & 20.5$\pm$0.3 & 5.6$\pm$0.9  & 7.0$\pm$0.5  & 15.2$\pm$0.5 \\
            BiC (2019)      & 21.2$\pm$0.3 & 36.1$\pm$1.3 & 42.5$\pm$1.2 & 10.2$\pm$0.9 & 18.9$\pm$0.3 & 25.2$\pm$0.6 \\
            SSIL (2021)     & 26.0$\pm$0.1 & 33.1$\pm$0.5 & 39.5$\pm$0.4 & 9.6$\pm$0.7  & 15.2$\pm$1.5 & 21.1$\pm$0.1 \\
            MOSE (2024)     & 37.4$\pm$0.3 & 47.0$\pm$0.4 & 55.6$\pm$0.4 & 24.7$\pm$0.5 & 32.4$\pm$0.3 & 40.6$\pm$0.5 \\
            \midrule
            \method (Ours)  & \textbf{39.4$\pm$0.4} & \textbf{47.6$\pm$0.1} & \textbf{56.3$\pm$0.1} & \textbf{28.1$\pm$0.2} & \textbf{34.8$\pm$0.2} & \textbf{41.4$\pm$0.3} \\
            \bottomrule
        \end{tabular}%
    }
\end{table*}

\begin{table*}[htbp]
\centering
\caption{Average Forgetting results (backward transfer) on Split CIFAR-100 and Split Tiny-ImageNet benchmarks.}
\label{tab:main_results_af_only}
\resizebox{0.97\textwidth}{!}{%
\begin{tabular}{l ccc ccc}
\toprule
\multirow{3}{*}{\textbf{Method}} & \multicolumn{3}{c}{\textbf{Split CIFAR-100 (10 tasks) - AF(\%) $\downarrow$}} & \multicolumn{3}{c}{\textbf{Split Tiny-ImageNet (100 tasks) - AF(\%) $\downarrow$}} \\
\cmidrule(lr){2-4} \cmidrule(lr){5-7}
& \textbf{M = 1k} & \textbf{M = 2k} & \textbf{M = 5k} & \textbf{M = 2k} & \textbf{M = 4k} & \textbf{M = 10k} \\
\midrule
AGEM (2019) & 77.6$\pm$2.0 & 76.9$\pm$1.5 & 78.3$\pm$1.2 & 73.9$\pm$0.2 & 77.9$\pm$0.2 & 74.1$\pm$0.3 \\
ER (2019) & 66.1$\pm$1.3 & 59.3$\pm$0.9 & 60.0$\pm$1.6 & 68.2$\pm$2.8 & 66.2$\pm$0.8 & 67.2$\pm$0.2 \\
MIR (2019) & 24.5$\pm$0.3 & 21.4$\pm$0.3 & 21.0$\pm$0.1 & 61.1$\pm$3.2 & 60.4$\pm$0.5 & 59.5$\pm$0.3 \\
GSS (2019) & 73.4$\pm$4.2 & 69.3$\pm$3.1 & 70.9$\pm$2.9 & 72.8$\pm$1.2 & 72.6$\pm$0.4 & 71.5$\pm$0.2 \\
ASER (2021) & 25.0$\pm$0.2 & 12.2$\pm$1.9 & 13.2$\pm$0.1 & 65.7$\pm$0.7 & 64.2$\pm$0.2 & 62.2$\pm$0.1 \\
ER-AML (2022) & 51.5$\pm$0.8 & 49.2$\pm$0.5 & 38.7$\pm$0.6 & 47.4$\pm$0.5 & 43.2$\pm$0.3 & 41.0$\pm$0.5 \\
GDumb (2020) & 16.7$\pm$0.5 & 17.6$\pm$0.2 & 16.8$\pm$0.4 & 15.9$\pm$0.5 & 14.6$\pm$0.3 & 11.7$\pm$0.2 \\
SCR (2021) & 17.5$\pm$0.2 & 11.6$\pm$0.5 & 5.6$\pm$0.4 & 19.4$\pm$0.3 & 15.4$\pm$0.3 & 14.9$\pm$0.7 \\
OCM (2022) & 12.2$\pm$0.3 & 8.5$\pm$0.3 & 4.5$\pm$0.3 & 23.5$\pm$1.9 & 21.0$\pm$0.3 & 18.6$\pm$0.5 \\
OnPro (2023) & 10.4$\pm$0.5 & 6.1$\pm$0.6 & 5.3$\pm$0.6 & 17.4$\pm$0.4 & 16.8$\pm$0.4 & 14.6$\pm$0.3 \\
GSA (2023) & 33.2$\pm$0.6 & 22.8$\pm$0.4 & 8.7$\pm$0.3 & 35.5$\pm$0.3 & 25.8$\pm$0.4 & 16.9$\pm$0.6 \\
DER++ (2020) & 43.4$\pm$0.2 & 44.0$\pm$1.9 & 25.8$\pm$3.5 & 67.2$\pm$1.7 & 63.6$\pm$0.3 & 55.2$\pm$0.7 \\
IL2A (2021) & 24.6$\pm$0.6 & 12.5$\pm$0.7 & 20.0$\pm$0.5 & 65.5$\pm$0.7 & 60.1$\pm$0.5 & 57.6$\pm$1.1 \\
Co2L (2021) & 16.9$\pm$0.4 & 16.6$\pm$0.6 & 9.9$\pm$0.7 & 60.5$\pm$0.5 & 52.5$\pm$0.9 & 42.5$\pm$0.8 \\
LUCIR (2019) & 60.0$\pm$0.1 & 47.5$\pm$0.9 & 44.3$\pm$0.7 & 46.4$\pm$0.7 & 42.2$\pm$0.9 & 37.6$\pm$0.7 \\
CCIL (2021) & 16.7$\pm$0.5 & 16.1$\pm$0.3 & 17.5$\pm$0.2 & 59.4$\pm$0.3 & 56.2$\pm$1.3 & 48.9$\pm$0.6 \\
BiC (2019) & 40.2$\pm$0.4 & 30.9$\pm$0.7 & 18.7$\pm$0.5 & 43.5$\pm$0.5 & 32.9$\pm$0.5 & 24.9$\pm$0.4 \\
SSIL (2021) & 40.1$\pm$0.5 & 33.9$\pm$1.2 & 21.7$\pm$0.8 & 44.4$\pm$0.7 & 36.6$\pm$0.7 & 29.0$\pm$0.7 \\
MOSE (2024)& 34.7$\pm$0.3 & 23.6$\pm$0.4 & 12.7$\pm$0.4 & 33.3$\pm$0.5& 22.1$\pm$0.4& 11.5$\pm$0.4\\
\midrule
\method (Ours)& 33.9$\pm$0.3 & 22.1$\pm$0.4 & 11.6$\pm$0.5 & 22.7$\pm$1.0 & 15.3$\pm$0.8 & 8.95$\pm$0.3 \\
\bottomrule
\end{tabular}
}
\end{table*}

\subsection{Ablation study}

In this section, we perform comprehensive ablation study and hyperparameter analysis to evaluate the effectiveness of the proposed method. We perform ablation study for the components in GRACE in Table \ref{tab:ablation_study}.

\begin{table*}[!htbp]
    \centering
    \begin{minipage}[t]{0.5\textwidth}
        \centering
        \scriptsize

        \caption{Comparison of memory efficiency on Split CIFAR-100 and Split Tiny-ImageNet.}
        \label{tab:memory_efficiency}
        \setlength{\tabcolsep}{3pt} 

        \begin{tabular}{@{}lcccc@{}}
            \toprule
            \multirow{2}{*}{\textbf{Method}} & \multicolumn{2}{c}{\textbf{Split CIFAR-100}} & \multicolumn{2}{c}{\textbf{Split Tiny-ImageNet}} \\
            \cmidrule(lr){2-3} \cmidrule(lr){4-5}
            & \textbf{M=200} & \textbf{M=500} & \textbf{M=500} & \textbf{M=1K} \\
            \midrule
            OCM (2022)      & 12.2$\pm$0.4 & 19.7$\pm$0.5 & 7.3$\pm$0.5  & 10.5$\pm$0.6 \\
            OnPro (2023)    & 14.1$\pm$0.9 & 21.5$\pm$1.4 & 7.2$\pm$0.4  & 10.2$\pm$0.3 \\
            GSA (2023)      & 14.9$\pm$0.3 & 22.9$\pm$0.2 & 10.4$\pm$0.3 & 14.8$\pm$0.2 \\
            MOSE (2024)     & 20.2$\pm$0.5 & 28.3$\pm$0.7 & 15.2$\pm$0.7 & 20.2$\pm$0.9 \\
            \midrule
            \method{} (Ours)& \textbf{21.5$\pm$0.5} & \textbf{29.8$\pm$0.6} & \textbf{16.6$\pm$0.6} & \textbf{21.7$\pm$0.8} \\
            \bottomrule
        \end{tabular}
    \end{minipage}%
    \hfill 
    \begin{minipage}[t]{0.5\textwidth}
        \centering
        
        \caption{Ablation study on Split-CIFAR-100 (with a buffer size of 1,000), showing accuracy drop when removing components.}
        \label{tab:ablation_study}
        \setlength{\tabcolsep}{3pt}

        \begin{tabular}{lc}
            \toprule
            \textbf{Method / Variation} & \textbf{Accuracy (\%)} \\
            \midrule
            \textbf{Main Model (Full)} & \textbf{39.4} \\
            \midrule
            \quad w/o Entropy Scaling   & 37.2 \\
            \quad w/o Adaptive Training & 38.9 \\
            \bottomrule
        \end{tabular}
    \end{minipage}

\end{table*}

\textbf{w/o Entropy Scaling:} Removing the entropy scaling mechanism causes the most significant performance degradation, with accuracy falling to 37.2\%. This confirms that adaptive scaling is the core contribution of our method. The performance suffers because, without scaling, any entropy regularization is applied uniformly. This is detrimental because early layers in a network are responsible for learning general, low-level features (e.g., edges, textures) that are common across many classes.

\textbf{w/o Adaptive Training:} Removing the adaptive training component resulted in a modest drop in accuracy from 39.4\% to 38.9\%. This is expected, as this mechanism is designed to intelligently manage the stability-plasticity trade-off. By adaptively reducing the regularization strength for layers that have already learned robust features for past tasks, it preserves critical knowledge. Removing this targeted intervention leads to slightly increased forgetting and a predictable drop in accuracy.

\textbf{Results with Vision Transformer Backbone}: We evaluate the effectiveness of our approach with vision transformer (ViT) \cite{dosovitskiy2021an} on attention layers. The results are presented in Appendix.

\includegraphics[width=\columnwidth]{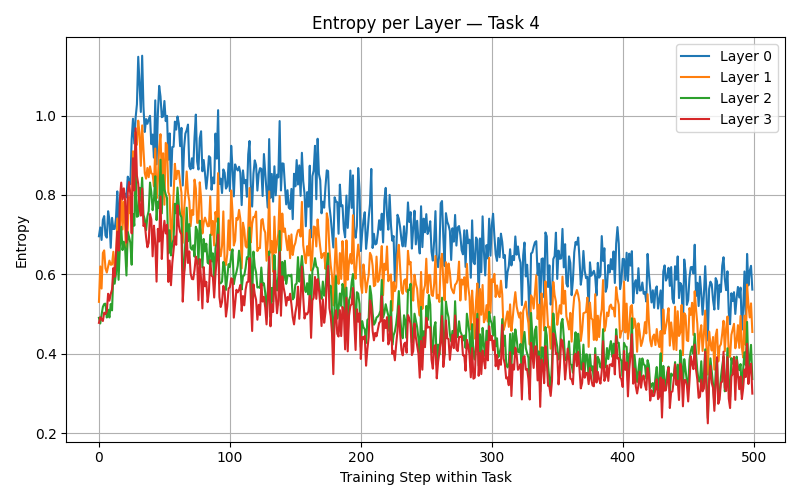}
\begin{figure}[htbp]
  \centering
  \caption{In a baseline model without our intervention, a significant entropy divergence emerges during training. Earlier layers consistently maintain high entropy, while deeper layers collapse to a low-entropy state, suggesting over-confidence. }
  \label{fig:left_image}
\end{figure}

\textbf{Results on CUB200 and CIFAR-10} The CUB200 and CIFAR-10 results are placed in the Appendix.

\includegraphics[width=\columnwidth]{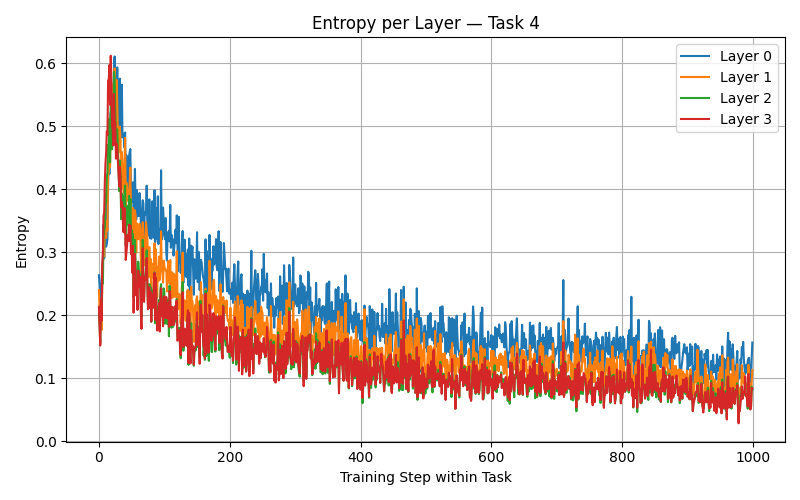}
\begin{figure}[htbp]
  \centering
  \caption{With entropy scaling, the entropies across all layers are successfully regularized. They converge towards a stable, medium-entropy state, showing that our method prevents individual layers from becoming either over-confident or under-confident. }
  \label{fig:right_image}
\end{figure}
 
\textbf{Hyperparameter Analysis (batch size, etc):} We do not add new hyperparameters, but the $\beta$ hyperparameter used in \citep{cha2021cpr} is changed from 0.05 to 0.005. We also test how different batch sizes affects our method. We find that the default batch size of 10 works well, while decreasing it to 5 or less will introduce too much noise depending on the specific samples that appear in a batch. Higher batch sizes of 20 or more maintain similar performance, as expected. Detailed results are in the Appendix. 

\textbf{Runtime Analysis:} 
Our method has a small computational overhead. Our method builds on top of MOSE, which takes twice as long as simpler methods such as SCR and ER, as each layer is evaluated on previous tasks. Hence, our method has a similar runtime. Compared to MOSE, the main overhead that is added is the computation and storage of the entropies of all the layers. On an average of ten runs on CIFAR-100 with buffer size 1000, the runtime increases from 20 minutes to 21 minutes. Detailed runtime metrics are included in the Appendix.

\section{Conclusion}

In this work, we introduced a novel, layer-wise feedback framework to mitigate catastrophic forgetting in continual learning. Our core method dynamically regularizes each layer by applying a penalty inversely proportional to its output entropy. Our approach is principled and modular and can be readily integrated into existing continual learning pipelines. We provide comprehensive theoretical analysis to show the benefits of our proposed approach. We have demonstrated its effectiveness through significant performance gains on various datasets.

{
    \small
    \bibliographystyle{ieeenat_fullname}
    \bibliography{main}
}

\appendix
\newpage
\section{Appendix}
\begin{theorem}[PAC-Bayes generalization with entropy control]
\label{thm:pacbayes-entropy}
Fix $\delta\in(0,1)$. For task $t$ with sample size $n_t\ge 2$, let
\begin{equation}
\begin{split}
    \mathcal{L}_t(\vtheta) &= R_t(\vtheta) + \lambda\,\Delta_t(\vtheta), \\
    R_t(\vtheta) &= \mathbb{E}_{(x,y)\sim\mathcal{D}_t}\!\big[\ell_t(f_\vtheta(x),y)\big]
\end{split}
\end{equation}

where $\ell_t\in[0,1]$ and $\Delta_t(\vtheta)=\sum_{\ell=1}^L \big(H_l(\vtheta)) - H_{\ell,t}^{*}\big)^2$.

Let $\hat{\mathcal{L}}_t(\vtheta)=\hat{R}_t(\vtheta)+\lambda\,\hat{\Delta}_t(\vtheta)$ be the empirical analogue on $n_t$ samples.
For any posterior $q_t$ absolutely continuous w.r.t.\ a prior $p_t$ (chosen before seeing the task-$t$ data), with probability at least $1-\delta$ over the sample,
\begin{equation}
\begin{split}
    \mathbb{E}_{\vtheta\sim q_t}\big[\mathcal{L}_t(\vtheta)\big]
    &\le
    \mathbb{E}_{\vtheta\sim q_t}\big[\hat{\mathcal{L}}_t(\vtheta)\big] \\
    &\quad +
    \sqrt{\frac{\mathrm{KL}(q_t\Vert p_t)+\ln\frac{2\sqrt{n_t}}{\delta}}{2(n_t-1)}} \\
    &\quad +
    \lambda\,\mathbb{E}_{\vtheta\sim q_t}\big[\Delta_t(\vtheta)\big].
\end{split}
\label{eq:main-pac}
\end{equation}

Moreover, if $p_t=\mathcal{N}(\vtheta_{t-1},\Sigma_p)$ and $q_t$ has mean $\bar\vtheta_t$ and covariance $\Sigma_q$,
\begin{equation}
\begin{split}
    \mathrm{KL}(q_t\Vert p_t)
    &\le
    \frac{\kappa_t}{2}\,\|\bar\vtheta_t-\vtheta_{t-1}\|^2
    \;+\;C_t, \\
    \text{where} \quad \kappa_t &:= \lambda_{\max}(\Sigma_p^{-1}), \\
    C_t &:= \tfrac{1}{2}\!\left(\mathrm{tr}(\Sigma_p^{-1}\Sigma_q)-k+\ln\frac{\det\Sigma_p}{\det\Sigma_q}\right).
\end{split}
\end{equation}
\end{theorem}

\begin{proof}[Proof of Theorem~\ref{thm:pacbayes-entropy}]

By a standard PAC-Bayes inequality for $[0,1]$-bounded losses (e.g., Seeger/McAllester form), with probability at least $1-\delta$ over the sample of size $n_t\ge 2$, for any posterior $q_t\ll p_t$,
\[
\mathbb{E}_{\vtheta\sim q_t}\!\big[R_t(\vtheta)\big]
\;\le\;
\mathbb{E}_{\vtheta\sim q_t}\!\big[\hat{R}_t(\vtheta)\big]
\;+\;
\sqrt{\frac{\mathrm{KL}(q_t\Vert p_t)+\ln\frac{2\sqrt{n_t}}{\delta}}{2(n_t-1)}}.
\tag{1}
\label{eq:pac-task}
\]
A succinct derivation is as follows. Let $r_\vtheta=\hat{R}_t(\vtheta)$ and $R_\vtheta=R_t(\vtheta)$. The “shift-of-measure” (Donsker–Varadhan) inequality implies that for any measurable $\phi$,
\[
\mathbb{E}_{q_t}[\phi(\vtheta)]
\;\le\;
\mathrm{KL}(q_t\Vert p_t)+\ln\mathbb{E}_{p_t}\!\big[e^{\phi(\vtheta)}\big].
\]
Apply this with $\phi(\vtheta)=\lambda\big(R_\vtheta-r_\vtheta\big)$ and bound the moment-generating function uniformly over $\vtheta$ by Hoeffding’s lemma for $[0,1]$-bounded losses (together with a union/intersection trick that yields the $n_t\!-\!1$ and the $\ln(2\sqrt{n_t}/\delta)$ refinements), then optimize over $\lambda>0$ to obtain~\eqref{eq:pac-task}.

\smallskip

By definition,
\[
\mathbb{E}_{q_t}\!\big[\mathcal{L}_t(\vtheta)\big]
\;=\;
\mathbb{E}_{q_t}\!\big[R_t(\vtheta)\big]
\;+\;
\lambda\,\mathbb{E}_{q_t}\!\big[\Delta_t(\vtheta)\big].
\]
Combining with~\eqref{eq:pac-task} gives
\begin{equation}
\begin{split}
    \mathbb{E}_{q_t}\!\big[\mathcal{L}_t(\vtheta)\big]
    &\le
    \mathbb{E}_{q_t}\!\big[\hat{R}_t(\vtheta)\big] \\
    &\quad +
    \sqrt{\frac{\mathrm{KL}(q_t\Vert p_t)+\ln\frac{2\sqrt{n_t}}{\delta}}{2(n_t-1)}} \\
    &\quad +
    \lambda\,\mathbb{E}_{q_t}\!\big[\Delta_t(\vtheta)\big].
\end{split}
\tag{2}
\label{eq:add-entropy}
\end{equation}
Now note $\hat{\mathcal{L}}_t(\vtheta)=\hat{R}_t(\vtheta)+\lambda\,\hat{\Delta}_t(\vtheta)\ge \hat{R}_t(\vtheta)$ since $\hat{\Delta}_t(\vtheta)\ge 0$. Hence
\[
\mathbb{E}_{q_t}\!\big[\hat{R}_t(\vtheta)\big]
\;\le\;
\mathbb{E}_{q_t}\!\big[\hat{\mathcal{L}}_t(\vtheta)\big],
\]
and substituting this into~\eqref{eq:add-entropy} yields the claimed bound~\eqref{eq:main-pac}. (This step deliberately avoids an empirical-process bound for $\Delta_t-\hat{\Delta}_t$; adding such a bound would replace the last $+\lambda\mathbb{E}_{q_t}[\Delta_t]$ term by $+\lambda\mathbb{E}_{q_t}[\hat{\Delta}_t]+$ a vanishing $O_\mathbb{P}(1/\sqrt{n_t})$ term.)

\smallskip

If $p_t=\mathcal{N}(\vtheta_{t-1},\Sigma_p)$ and $q_t$ has mean $\bar\vtheta_t$ and covariance $\Sigma_q$, the Gaussian KL identity gives
\begin{equation}
\begin{split}
    \mathrm{KL}(q_t\Vert p_t)
    = \tfrac12\bigg(
    &\mathrm{tr}(\Sigma_p^{-1}\Sigma_q) \\
    &+(\bar\vtheta_t-\vtheta_{t-1})^\top \Sigma_p^{-1}(\bar\vtheta_t-\vtheta_{t-1}) \\
    &-k
    +\ln\frac{\det\Sigma_p}{\det\Sigma_q}
    \bigg).
\end{split}
\end{equation}
Using $(\bar\vtheta_t-\vtheta_{t-1})^\top \Sigma_p^{-1}(\bar\vtheta_t-\vtheta_{t-1})
\le \lambda_{\max}(\Sigma_p^{-1})\,\|\bar\vtheta_t-\vtheta_{t-1}\|^2$ gives the stated bound with $\kappa_t=\lambda_{\max}(\Sigma_p^{-1})$ and the remaining terms absorbed into $C_t$.
\end{proof}

\begin{theorem}[Forgetting bound via parameter drift]
\label{thm:forgetting-drift}
Let $\mathcal{F}_{s\to t}:=\mathcal{L}_s(\vtheta_t)-\mathcal{L}_s(\vtheta_s)$ for $1\le s<t\le T$. 
Assume:

\begin{enumerate}[label=(A\arabic*)]
\item \label{A:boundedloss} $\ell_s\in[0,1]$ and along the optimization trajectory the population objective $\mathcal{L}_s$ has bounded gradient: $\sup_{\vtheta\in\Gamma}\|\nabla \mathcal{L}_s(\vtheta)\|\le L_s$, where $\Gamma$ contains $\{\vtheta_k\}_{k=s}^t$ and the line segments between successive iterates.
\item \label{A:entropy-lip} (Entropy Lipschitzness) For each layer $\ell$, $H_l(\vtheta)$ is Lipschitz in $\vtheta$ with constant $c_\ell$ along $\Gamma$.
\item \label{A:local-strong} (Local strong convexity/PL for task $k$) The empirical objective $J_k(\vtheta):=\hat{\mathcal{L}}_k(\vtheta)=\hat{R}_k(\vtheta)+\lambda\,\hat{\Delta}_k(\vtheta)$ is $\mu_k$-strongly convex on the segment between $\vtheta_{k-1}$ and $\vtheta_k$, i.e.,
\[
\big\langle \nabla J_k(\vtheta)-\nabla J_k(\vtheta'),\,\vtheta-\vtheta'\big\rangle \;\ge\; \mu_k\,\|\vtheta-\vtheta'\|^2.
\]
\end{enumerate}
Then
\begin{equation}
\begin{split}
    \mathcal{F}_{s\to t}
    &\le
    L_s\,\sum_{k=s+1}^{t}
    \frac{1}{\mu_k}\Big(
    \big\|\nabla \hat{R}_k(\vtheta_{k-1})\big\| \\
    &\qquad\quad +
    2\lambda\,C_\Delta\,\sqrt{\hat{\Delta}_k(\vtheta_{k-1})}
    \Big), \\
    \text{where} \quad C_\Delta &:= \Big(\sum_{\ell=1}^L c_\ell^2\Big)^{1/2}.
\end{split}
\label{eq:forgetting-final}
\end{equation}
Equivalently, since $\|\nabla \hat{R}_k\|\le \|\nabla \hat{\mathcal{L}}_k\|+2\lambda C_\Delta\sqrt{\hat{\Delta}_k}$,
\[
\mathcal{F}_{s\to t}
\;\le\;
L_s\,\sum_{k=s+1}^{t}
\frac{1}{\mu_k}\Big(
\big\|\nabla \hat{\mathcal{L}}_k(\vtheta_{k-1})\big\|
\;+\;
4\lambda\,C_\Delta\,\sqrt{\hat{\Delta}_k(\vtheta_{k-1})}
\Big).
\]
\end{theorem}

\begin{proof}[Proof of Theorem~\ref{thm:forgetting-drift}]
\textbf{ Reduce forgetting to parameter displacement.}
By the fundamental theorem of calculus along the line segment from $\vtheta_s$ to $\vtheta_t$ and the bounded-gradient assumption~\ref{A:boundedloss},
\begin{equation}
\begin{split}
    \mathcal{F}_{s\to t}
    &=
    \int_0^1
    \left\langle \nabla \mathcal{L}_s\big(\vtheta_s+\tau(\vtheta_t-\vtheta_s)\big),\,\vtheta_t-\vtheta_s\right\rangle d\tau \\
    &\le
    \Big(\sup_{\vtheta\in\Gamma}\|\nabla \mathcal{L}_s(\vtheta)\|\Big)\,\|\vtheta_t-\vtheta_s\| \\
    &\le
    L_s\,\|\vtheta_t-\vtheta_s\|.
\end{split}
\end{equation}
By the triangle inequality,
\[
\|\vtheta_t-\vtheta_s\|
\;\le\;
\sum_{k=s+1}^{t}\|\vtheta_k-\vtheta_{k-1}\|.
\]
Hence
\begin{equation}
\label{eq:forgetting-step1}
\mathcal{F}_{s\to t}
\;\le\;
L_s\,\sum_{k=s+1}^{t}\|\vtheta_k-\vtheta_{k-1}\|.
\end{equation}

\smallskip
\textbf{Bound each inter-task jump by the local geometry of $J_k$.}
Since $\vtheta_k$ is a (local) minimizer or a first-order stationary point of $J_k$ on task $k$, $\nabla J_k(\vtheta_k)=0$.
By strong convexity/strong monotonicity along the segment (Assumption~\ref{A:local-strong}) and Cauchy–Schwarz,
\begin{equation}
\begin{split}
    \mu_k\,\|\vtheta_k-\vtheta_{k-1}\|
    &\le
    \big\|\nabla J_k(\vtheta_{k-1})-\nabla J_k(\vtheta_k)\big\| \\
    &=
    \big\|\nabla J_k(\vtheta_{k-1})\big\|.
\end{split}
\end{equation}
Therefore
\begin{equation}
\begin{split}
\label{eq:jump-vs-grad}
    \|\vtheta_k-\vtheta_{k-1}\|
    &\le
    \frac{1}{\mu_k}\,\big\|\nabla J_k(\vtheta_{k-1})\big\| \\
    &=
    \frac{1}{\mu_k}\,\big\|\nabla \hat{R}_k(\vtheta_{k-1})+\lambda\,\nabla \hat{\Delta}_k(\vtheta_{k-1})\big\|.
\end{split}
\end{equation}

\smallskip
\textbf{Control the entropy-gradient by entropy deviation.}
Write $\hat{\Delta}_k(\vtheta)=\sum_{\ell=1}^L d_{\ell,k}(\vtheta)^2$ 
with $d_{\ell,k}(\vtheta)=H_l(\vtheta)) - H_{\ell,k}^{*}$.

By the chain rule,
\[
\nabla \hat{\Delta}_k(\vtheta)
=
2\sum_{\ell=1}^L d_{\ell,k}(\vtheta)\,\nabla H_l(\vtheta)).
\]
By Assumption~\ref{A:entropy-lip} and Rademacher’s theorem (Lipschitz $\Rightarrow$ a.e.\ differentiable with gradient norm bounded by the Lipschitz constant) we have $\|\nabla H_l(\vtheta)\|\le c_\ell$ along $\Gamma$.
Thus, by Cauchy–Schwarz,

\begin{equation}
\begin{split}
\label{eq:grad-delta-bound}
    \|\nabla \hat{\Delta}_k(\vtheta)\|
    &\le
    2\Big(\sum_{\ell=1}^L d_{\ell,k}(\vtheta)^2\Big)^{1/2} 
     \cdot \Big(\sum_{\ell=1}^L c_\ell^2\Big)^{1/2} \\
    &=
    2\,C_\Delta\,\sqrt{\hat{\Delta}_k(\vtheta)}.
\end{split}
\end{equation}

Evaluating at $\vtheta=\vtheta_{k-1}$ yields
\[
\big\|\nabla J_k(\vtheta_{k-1})\big\|
\;\le\;
\big\|\nabla \hat{R}_k(\vtheta_{k-1})\big\|
\;+\;
2\lambda\,C_\Delta\,\sqrt{\hat{\Delta}_k(\vtheta_{k-1})}.
\]
Combine this estimate with~\eqref{eq:jump-vs-grad} to obtain
\[
\|\vtheta_k-\vtheta_{k-1}\|
\;\le\;
\frac{1}{\mu_k}\Big(
\big\|\nabla \hat{R}_k(\vtheta_{k-1})\big\|
\;+\;
2\lambda\,C_\Delta\,\sqrt{\hat{\Delta}}_k(\vtheta_{k-1})
\Big).
\]

\smallskip
Plug the last inequality into~\eqref{eq:forgetting-step1} to conclude
\begin{equation}
\begin{split}
    \mathcal{F}_{s\to t}
    &\le
    L_s\,\sum_{k=s+1}^{t}
    \frac{1}{\mu_k}\Big(
    \big\|\nabla \hat{R}_k(\vtheta_{k-1})\big\| \\
    &\qquad\quad +
    2\lambda\,C_\Delta\,\sqrt{\hat{\Delta}_k(\vtheta_{k-1})}
    \Big),
\end{split}
\end{equation}
which is~\eqref{eq:forgetting-final}. Finally, since
$\|\nabla \hat{R}_k(\vtheta_{k-1})\|
\le \|\nabla \hat{\mathcal{L}}_k(\vtheta_{k-1})\|
+ \lambda\|\nabla \hat{\Delta}_k(\vtheta_{k-1})\|
\le \|\nabla \hat{\mathcal{L}}_k(\vtheta_{k-1})\|
+ 2\lambda C_\Delta \sqrt{\hat{\Delta}_k(\vtheta_{k-1})}$ by~\eqref{eq:grad-delta-bound}, the equivalent variant stated in the theorem also follows.
\end{proof}

\section{Bayesian Derivation of Adaptive Entropy Scaling}

We model the adaptive entropy scaling factor \( \gamma_\ell \) in layer \( \ell \) as a latent variable in a Bayesian framework. The goal is to infer a suitable regularization strength for each layer, based on its entropy \( H_\ell \). We use variational inference to approximate the posterior distribution over \( \gamma_\ell \) given the observed entropy.

\subsection{Generative Model}

Let \( \gamma_\ell \in \mathbb{R}_+ \) be the entropy regularization strength for layer \( \ell \), and \( H_\ell \in \mathbb{R}_+ \) the entropy observed at that layer.
We assume a Gaussian noise model for entropy given regularization:
\[
H_\ell = H^* + \frac{c}{\gamma_\ell} + \varepsilon_\ell, \quad \varepsilon_\ell \sim \mathcal{N}(0, \sigma^2)
\]
Thus, the likelihood becomes:
\[
p(H_\ell \mid \gamma_\ell) = \mathcal{N}\left(H_\ell \mid H^* + \frac{c}{\gamma_\ell}, \sigma^2 \right)
\]

We place a log-normal prior on \( \gamma_\ell \) to enforce positivity:
\[
p(\gamma_\ell) = \text{LogNormal}(\mu_0, \tau^2) \quad \Rightarrow \quad \log \gamma_\ell \sim \mathcal{N}(\mu_0, \tau^2)
\]

\subsection{Posterior Distribution}

The posterior over \( \gamma_\ell \) is intractable, so we approximate it via variational inference. Let the variational posterior be:
\[
q_\phi(\gamma_\ell) = \text{LogNormal}(\mu_\phi, \sigma_\phi^2) \quad \Rightarrow \quad \log \gamma_\ell \sim \mathcal{N}(\mu_\phi, \sigma_\phi^2)
\]

We optimize the evidence lower bound (ELBO):
\[
\mathcal{L}(\phi) = \mathbb{E}_{\gamma_\ell \sim q_\phi} \left[
\log p(H_\ell \mid \gamma_\ell)
+ \log p(\gamma_\ell)
- \log q_\phi(\gamma_\ell)
\right]
\]

We place a log-normal prior on \( \gamma_\ell \) to ensure positivity and induce regularization:

\[
\gamma_\ell \sim \text{LogNormal}(\mu_0, \tau^2) \quad \text{which implies} \quad \log \gamma_\ell \sim \mathcal{N}(\mu_0, \tau^2)
\]

The density of the log-normal distribution is:

\[
p(\gamma_\ell) = \frac{1}{\gamma_\ell \tau \sqrt{2\pi}} \exp\left( -\frac{(\log \gamma_\ell - \mu_0)^2}{2\tau^2} \right)
\]

Taking the logarithm gives the log-prior:

\[
\log p(\gamma_\ell) = -\frac{(\log \gamma_\ell - \mu_0)^2}{2\tau^2} - \log \gamma_\ell + \text{const}
\]

Substituting the log-densities:

\begin{align*}
\log p(H_\ell \mid \gamma_\ell) &= -\frac{1}{2\sigma^2} \left(H_\ell - H^* - \frac{c}{\gamma_\ell} \right)^2 + \text{const} \\
\log p(\gamma_\ell) &= -\frac{1}{2\tau^2} (\log \gamma_\ell - \mu_0)^2 - \log \gamma_\ell + \text{const} \\
\log q_\phi(\gamma_\ell) &= -\frac{1}{2\sigma_\phi^2} (\log \gamma_\ell - \mu_\phi)^2 - \log \gamma_\ell + \text{const}
\end{align*}

Therefore, the ELBO becomes:

\begin{align*}
\mathcal{L}(\phi) = \mathbb{E}_{\gamma_\ell \sim q_\phi} \bigg[
& -\frac{1}{2\sigma^2} \left(H_\ell - H^* - \frac{c}{\gamma_\ell} \right)^2 \\
& -\frac{1}{2\tau^2} (\log \gamma_\ell - \mu_0)^2 - \log \gamma_\ell \\
& + \frac{1}{2\sigma_\phi^2} (\log \gamma_\ell - \mu_\phi)^2 + \log \gamma_\ell 
\bigg] + \text{const}
\end{align*}

Observe that the \( -\log \gamma_\ell + \log \gamma_\ell \) terms cancel out, simplifying the ELBO.

\subsection{Reparameterization Trick}

We use the reparameterization trick for gradient estimation:

\[
\log \gamma_\ell = \mu_\phi + \sigma_\phi \cdot \epsilon, \quad \epsilon \sim \mathcal{N}(0, 1)
\]

We can now estimate the ELBO and its gradient using Monte Carlo samples of \( \epsilon \).

\subsection{Deterministic Approximation: GRACE Scaling Rule}

To avoid optimizing \( \mathcal{L}(\phi) \) explicitly during training, we approximate the \textit{posterior mean}:

\[
\hat{\gamma}_\ell = \mathbb{E}_{q_\phi}[\gamma_\ell] = \exp\left( \mu_\phi + \frac{\sigma_\phi^2}{2} \right)
\]

Assuming a small variance \( \sigma_\phi^2 \approx 0 \), we approximate:

\[
\hat{\gamma}_\ell \approx \exp(\mu_\phi)
\]

Empirically, we set:

\[
\mu_\phi \approx \tanh\left( z_\ell \right), \quad \text{where } z_\ell = \frac{H_\ell - \mu_H}{\sigma_H}
\]

This yields the final approximation used in GRACE:

\[
\gamma_\ell \approx \exp\left( \tanh\left( \frac{H_\ell - \mu_H}{\sigma_H} \right) \right)
\]

\section{Ablations}
\subsection{ViT and CUB200}
We test the efficacy of our method on different backbone architectures, namely the vision transformer, and we also test on the more challenging CUB-200 dataset, which requires the model to distinguish between fine grained details. The images are in 224x224 resolution and we use vit-base-patch16-224. As the transformer architecture maintains a constant dimension throughout the network, we can do away with the downsampling and upsampling required with CNNs. We split the 12 layers into 4 blocks, and apply our method to each block attaching classification head to each of these blocks. These experiments were run on a NVIDIA RTX A6000. We use the same experimental setup as \citep{Wei2025Online} and adapt their results. In table \ref{tab:cub200}, we find that our method achieves better results than many previous methods, revealing that our method can be freely applied to many different types of backbone architectures. 

\begin{table}[htbp]
    \centering
    \begin{tabular}{lcc}
        \toprule
        \textbf{Method} & \textbf{$A_{Final}$ ($\uparrow$)} & \textbf{Forgetting ($\downarrow$)} \\
        \midrule
        AGEM          & $10.84 \pm 1.57$ & $47.79 \pm 0.04$ \\
        ER          & $31.66 \pm 0.83$ & $14.23 \pm 0.07$ \\
        EWC++    & $26.14 \pm 3.46$ & $47.69 \pm 0.07$ \\
        MIR          & $31.64 \pm 2.97$ & $23.43 \pm 0.05$ \\
        GDumb       & $9.09 \pm 1.03$  & - \\
        PCR         & $41.11 \pm 1.43$ & $29.64 \pm 1.20$ \\
        DER++         & $26.61 \pm 1.27$ & $32.16 \pm 0.55$ \\
        LODE (DER++) & $39.20 \pm 4.25$ & $41.64 \pm 3.59$ \\
        EMA (DER++)  & $35.26 \pm 3.31$ & $25.55 \pm 3.35$ \\
        EMA (RAR)     & $33.34 \pm 1.11$ & $28.68 \pm 0.56$ \\
        Online-LoRA       & $41.46 \pm 0.31$ & $13.64 \pm 0.68$ \\
       \textbf{GRACE (Ours)  }    & $ \textbf{43.53} \pm \textbf{1.08}$ & $\textbf{9.21} \pm \textbf{2.14}$ \\
        \bottomrule
    \end{tabular}
    \caption{Results on CUB-200}
    \label{tab:cub200}
\end{table}

\subsection{Runtime and Memory}
We compare the runtime of our method and baseline methods and report the results in table \ref{tab:training_cost}. As noted before, our method takes twice as long as methods such as SCR, but we believe the accuracy gain makes up for it. The increase in the number of parameters 
comprises a much smaller percentage, and is expected as we store intermediate activations. Compared against the competitive MOSE method, we are able to see performance gains without a noticeable increase in runtime or memory usage. 

\begin{table}[htbp]
    \centering
    \caption{Training cost comparison for CIFAR-100 with buffer size = 5k.}
    \label{tab:training_cost}
    \begin{tabular}{lcccc}
        \toprule
        Model & Time & \# Params & ACC $\uparrow$ & AF $\downarrow$ \\
        \midrule
        ER     & 3.3m  & 11.3M & $37.7\pm0.4$ & $45.8\pm1.2$ \\
        SCR      & 8.0m & 11.3M & $38.8\pm0.8$ & $10.4\pm0.8$ \\
        MOSE   & 16.7m & 13.0M & $54.9\pm0.4$ & $12.7\pm0.4$ \\
        GRACE (Ours)   & 16.8m & 13.0M & $56.3\pm0.1$ & $11.6\pm0.5$ \\
        \bottomrule
    \end{tabular}
\end{table}

\subsection{Hyperparameter Analysis}
We investigate the effect of batch size on performance in table \ref{tab:cifar100_batch_experiments}. The standard batch size for online continual learning experiments is 10, and we also test on the strict scenario of having batch size 1. The decrease in performance when the batch size is increased to 20 can be attributed to the fact that although the gradient directions may be less noisy, there are fewer update steps for the model to learn the current task.

\begin{table}[htbp]
    \centering
    \caption{CIFAR-100 Accuracy and Average Forgetting across varying batch sizes.}
    \label{tab:cifar100_batch_experiments}
    \begin{tabular}{lcc}
        \toprule
        \textbf{Batch Size} & \textbf{Acc (\%)} & \textbf{AF (\%)} \\ 
        \midrule
        1            &$21.1 \pm 1.4$ & $57.2 \pm 0.9$  \\
        5            &$32.7 \pm 0.8$ & $42.8 \pm 0.4$ \\
        10 (default) & $39.4 \pm 0.4$  & $33.9 \pm 0.3$\\
        20           & $36.1 \pm 0.6$ & $37.7 \pm 0.6$  \\
        \bottomrule
    \end{tabular}
\end{table}

\end{document}